\documentclass[journal]{IEEEtran}
\usepackage{times}
\usepackage{epsfig}
\usepackage{graphicx}
\usepackage{amsmath}
\usepackage{amssymb}
\usepackage{amsthm}
\usepackage{multirow}
\usepackage{color}
\usepackage{caption}
\usepackage{algorithm}
\usepackage{algorithmic}
\newcommand{\tabincell}[2]{\begin{tabular}{@{}#1@{}}#2\end{tabular}}

\ifCLASSINFOpdf
\else
\fi
\begin{document}
%
\title{Deep Manifold-to-Manifold Transforming Network}
%
%
%

\author{Tong~Zhang,
        Wenming~Zheng*,~\IEEEmembership{Member,~IEEE,}
        Zhen~Cui*,~
        Chaolong Li
\thanks{Tong Zhang is with the Key Laboratory of Child Development and Learning Science of Ministry of Education, and the Department
of Information Science and  Engineering, Southeast University, China. \protect\\
E-mail: tongzhang@seu.edu.cn.}
\thanks{Wenming Zheng and Chaolong Li are with the Key Laboratory of Child Development and Learning Science of Ministry of Education, Research Center for Learning Science,
Southeast University, Nanjing, Jiangsu 210096, China.\protect\\ 
E-mail: wenming\_zheng@seu.edu.cn, lichaolong@seu.edu.cn.}
\thanks{Zhen Cui is with the School of Computer Science and Engineering, Nanjing University of Science and Technology, Nanjing, China.\protect\\ 
E-mail: zhen.cui@njust.edu.cn.\protect\\
Asterisk indicates corresponding author.\protect\\
}
}

%
%

\markboth{Journal of \LaTeX\ Class Files,~Vol.~14, No.~8, August~2015}%
{Shell \MakeLowercase{\textit{et al.}}: Bare Demo of IEEEtran.cls for IEEE Journals}
%



\maketitle

\begin{abstract}
Symmetric positive definite (SPD) matrices (e.g., covariances, graph Laplacians, etc.) are widely used to model the relationship of spatial or temporal domain. Nevertheless, SPD matrices are theoretically embedded on Riemannian manifolds. In this paper, we propose an end-to-end deep manifold-to-manifold transforming network (DMT-Net) which can make SPD matrices flow from one Riemannian manifold to another more discriminative one. To learn discriminative SPD features characterizing both spatial and temporal dependencies, we specifically develop three novel layers on manifolds: (i) the local SPD convolutional layer, (ii) the non-linear SPD activation layer, and (iii) the Riemannian-preserved recursive layer. The SPD property is preserved through all layers without any requirement of singular value decomposition (SVD), which is often used in the existing methods with expensive computation cost.
Furthermore, a diagonalizing SPD layer is designed to efficiently calculate the final metric for the classification task. To evaluate our proposed method, we conduct extensive experiments on the task of action recognition, where input signals are popularly modeled as SPD matrices. The experimental results demonstrate that our DMT-Net is much more competitive over state-of-the-art.
\end{abstract}

\begin{IEEEkeywords}
Riemannian manifold, SPD matrix, Deep learning, Action Recognition, Emotion Recognition.
\end{IEEEkeywords}

%
\IEEEpeerreviewmaketitle

\section{Introduction}
High-order statistics feature learning is one of the most active areas in pattern recognition. Especially, in the past few decades, covariance matrix was proposed as a generic descriptor. As the robustness to rotation, scale change and outlier, region covariance descriptors have achieved promising performance in object detection~\cite{tuzel2006region},  texture classification~\cite{Harandi2012Sparse} and tracking ~\cite{porikli2006covariance}.  When dealing with skeleton-based human  action recognition, a large amount of variants are evolved from the covariance theory. For example, to represent human skeleton graph in action recognition, a kernelized version named kernel-matrix-based (KMB) descriptor~\cite{wang2015beyond} is proposed to depict the relationship between skeletal joints. In~\cite{hussein2013human}, Hussein et al. computes the statistical covariance of 3D Joints (Cov3DJ)  as spatio-temporal SPD features to encode the relationship between joint movement meanwhile takes the temporal variation of action sequences into account. The driving forces to this trend are the powerful representation ability and the behind fundamental mathematical theory of Riemannian manifold spanned by symmetric positive definite (SPD) matrices, of which covariance is a special case.

For SPD descriptors, two crucial issues should be well solved.  The first one is how to learn more discriminative features from SPD matrix space. Several approaches, such as manifold-to-manifold transformation~\cite{harandi2014manifold} and locally linear embedding~\cite{goh2008clustering},  attempt to seek for the optimal SPD embedding matrices on Riemannian manifold. Inspired by deep learning, more recently, Huang et al.~\cite{huang2017riemannian} proposed a Riemannian network to extract high-level features from SPD matrices by designing the bilinear mapping (BiMap) layer and eigenvalue rectification (ReEig) layer. The second one is how to define the metric of SPD matrices. As SPD matrices lie on Riemannian manifold rather than Euclidean space, directly applying the algorithm designed in Euclidean geometry to SPD matrices may lead poor performances. To address this problem, some approximate metrics are proposed under the framework of manifold. Especially, Log-Euclidean metric~\cite{tuzel2008pedestrian,tosato2010multi,edelman1998geometry} flattens Riemannian manifold to tangent space so that numerous conventional methods designed in Euclidean space can be used. However, this process inevitably need to calculate matrix logarithm, which has high computation cost due to the requirement of SVD.

Furthermore, the modeling ability of temporal dynamics need to be enhanced to reduce the obscure of a single matrix descriptor (e.g., covariance) for a sequence. Recently recursive learning~\cite{cho2014learning} and convolutional neural network (CNN)~\cite{krizhevsky2012imagenet} have obtained the breakthrough successes, but they only work in Euclidean space, and generalizing them to manifolds should have a constructive value to the SPD descriptor. To this end, several crucial issues need to be solved:
\begin{enumerate}
\item How to perform local convolutional filtering on SPD matrices with manifold preservation?
\item How to perform recursive learning along Riemannian manifolds so as to model temporal dynamics?
\item How to avoid  expensively-computational SVD during the computation of metrics in order to reduce computation cost?
\end{enumerate}

In this paper, we propose a novel deep manifold-to-manifold transforming network (DMT-Net) to address all above issues. To implement local convolutional filtering on SPD matrices, we specifically design an SPD layer by constraining those filters to be SPD. Under this constraint, we theoretically prove manifold preservation after convolutional filtering. To enhance the flexibility, we also design a non-linear activation layer with the guarantee of manifold preservation, which only need perform element-wise operation and thus does not require SVD.
To model sequence dynamics, we specifically design an manifold-preserved recursive layer to encode sequentially SPD matrices of segmented subclips. In metric computation, we design a diagonalizing layer to convert each SPD map into a positive diagonal matrix, which makes log-Euclidean matric be efficiently calculated and avoids high-computational SVD.  All these are integrated together and jointly trained for  recognition. To evaluate the effectiveness of the proposed method, we conduct experiments on various datasets of action recognition. The experimental results show that our method outperforms those state-of-the-arts.

In summary, our main contributions are four folds:
\begin{enumerate}
\item we propose a true local SPD convolutional layer with a theoretical guarantee of Riemannian manifold preservation, while the literature~\cite{huang2017riemannian} only employed a bilinear operation (referred to the BiMap layer).
\item we design a non-linear activation layer to avoid SVD, with a theoretical proof of manifold preservation, while the literature~\cite{huang2017riemannian} still need SVD (referred to the ReEig layer) because its framework is based on the standard logarithm operation of matrix.
\item we design a manifold-preserved recurisive layer to encode temporal dependencies of SPD matrices.
\item we develop an elegant diagonalizing trick to bypass the high computation of SVD when applying log-Euclidean mapping, while  the literature~\cite{huang2017riemannian} employed the standard log-Euclidean metric need SVD.
\end{enumerate}

\section{Related work}

The most related work includes two aspects: SPD matrix descriptors and skeleton based action recognition. Below we briefly review them.

SPD matrices~\cite{hussein2013human,hussein2013human,wang2015beyond,koles1990spatial} have been widely used as features in different patten recognition tasks.  For instance,  as a special case of SPD matrix, covariance matrices are used to encode important relationship among regions in object detection~\cite{tuzel2006region}, object tracking~\cite{porikli2006covariance}, face recognition~\cite{pang2008gabor} and so on. As SPD matrices lie on Riemannian manifolds, most algorithms attempted to extract discriminative features by operating SPD descriptors on manifolds, such as Laplacian Eigenmaps (LE)~\cite{belkin2003laplacian}, Locally Linear Embedding (LLE)~\cite{goh2008clustering} and  manifold-to-manifold transformation~\cite{harandi2014manifold}. To measure distances between two points of manifold, various Riemannian metrics have been proposed such as  affine-invariant metric~\cite{pennec2006riemannian}, Log-Euclidean metric~\cite{arsigny2006log}, et al. Further, kernelized metric versions are also developed in~ \cite{wang2015discriminant,jayasumana2013kernel}, which  map  data to an RKHS with a family of positive definite kernels.


As an important research field of computer vision, skeleton based action recognition has drawn wide attention. In recent years, various algorithms have been proposed in previous literatures~\cite{amor2016action,Chaudhry2013Bio,devanne20153,du2015hierarchical,evangelidis2014skeletal,gowayyed2013histogram,Ofli2012Sequence,seidenari2013recognizing,tao2015moving,vemulapalli2014human}. Some of them focus on modeling temporal action dynamics of 3D  joint locations.  For instance, Xia et al.~\cite{xia2012view} learned high level visual words from skeleton data and employed hidden Markov models (HMMs) to character temporal evolutions of those visual words. In~\cite{baccouche2011sequential,du2015hierarchical}, to better model temporal dependencies,  recurrent neural networks (RNNs) are employed to capture the temporal variations of trajectories.  On the other hand, some other literatures focus on describing the similarities between joints by constructing high-order statistics features lying on manifolds, e.g. covariance matrix and its evolved versions. Based on these high-order features, featuring learning methods  are performed to learn more discriminative descriptors while respecting the underlying manifold structure. In~\cite{wang2015beyond}, an SPD descriptor named kernel-matrix-based (KMB) representation is proposed as an evolution of  covariance matrix to describe motion,  then  support vector machine (SVM)  with a log-Euclidean kernel is employed for classification.  In ~\cite{harandi2014manifold}, Mehrtash et al. perform manifold-to-manifold learning to obtain a lower-dimensional while more discriminative SPD feature from the original SPD descriptor named covariance of 3D joints (Cov3DJ)~\cite{hussein2013human}. Motivated by the progress of deep learning, especially, a deep neural network architecture of SPD matrices is proposed in the literature~\cite{huang2017riemannian} recently. As illustrated in the introduction, ours is great different from this work.

\begin{figure*}[!t]
\centering
\includegraphics[width=6.8in]{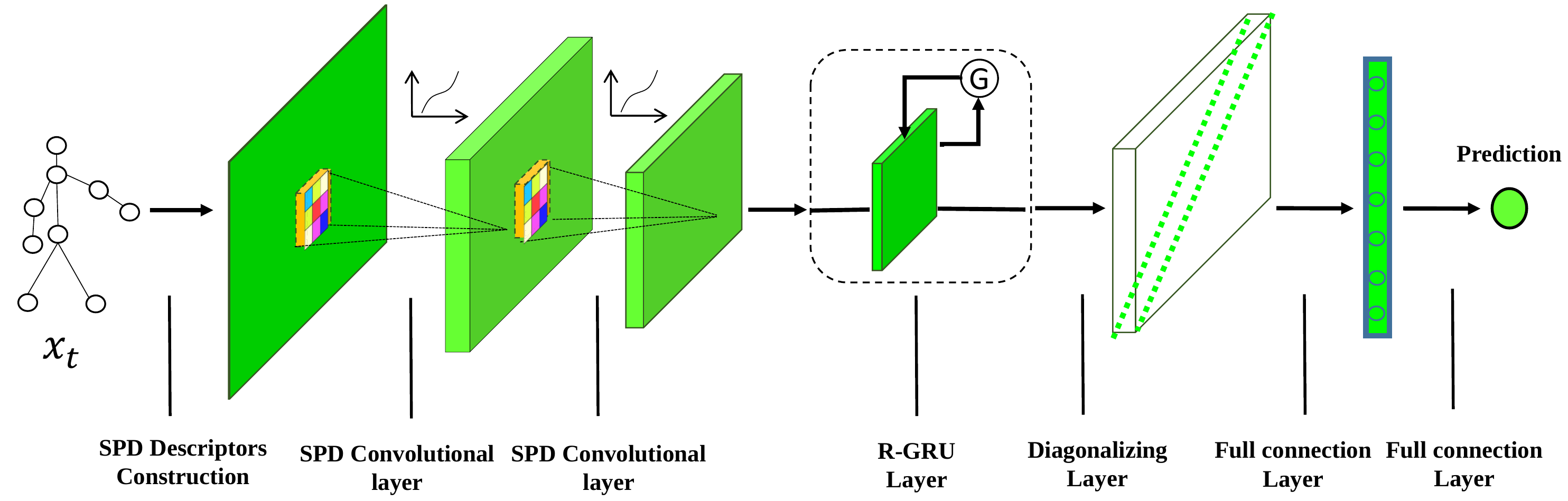}
\caption{The architecture of DMT network for action recognition. The raw spatio-temporal SPD features describing the skeleton-based actions are fed into the network. The SPD convolutional layers (Section \ref{sec:spd_cnn}) includes local SPD convolutional filtering and non-linear SPD activation charactering spatial dependencies. The SPD recursive layer captures temporal dependencies of sequential SPD descriptors with manifold preservation. The diagonalizing layer (Section \ref{sec:diag}) converts SPD matrices to the specific diagonalized SPD matrices so as to implement an efficient metric computation in the next layer. With the theoretical guarantee, the matrix descriptors flow from one Riemannian manifold to another Riemannian manifold for the sake of searching more discriminative manifold spaces.}
\label{fig_DMT}
\end{figure*}

\section{The Proposed Network}
\label{DMT-network}

In this section, the whole architecture of DMT-Net is firstly overviewed. Then, several specific-designed layers including SPD convolution, SPD non-linear activation, diagonalizing and SPD recursive layer are introduced in the sequent subsections in detail.

\subsection{Overview}

Fig.~\ref{fig_DMT} illustrates the whole architecture of our DMT-Net, which takes raw SPD matrices/descriptors as the inputs and layerwisely extracts discriminative features from one manifold to another manifold. In DMT-Net, we specifically develop several layers to make features still evolved on manifolds, which completely avoid SVD operations for SPD matrices.

Local SPD convolutional layer performs locally convolutional filtering on SPD descriptors extracted from a subclip like the standard CNN. But differently, to preserve SPD of the transformation, we need constrain the convolutional kernels into an SPD space rather than a general Euclidean space. Further, the non-linear activation function is specifically designed to satisfy manifold preservation, which is different from the literature~\cite{huang2017riemannian} using SVD on SPD matrices. The detailed introduction is described in Section~\ref{sec:spd_cnn}.

 After consecutively stacking SPD convolutional layers like those conventional CNNs, we can abstract discriminative SPD features from each subclip, which models the spatial information on manifold. Then for each action sequence, the learned SPD features of each subclip are fed into the designed SPD recursive layer for modeling temporal dependencies. Different from conventional RNN in Euclidean space, this recursive layer makes each state still flow on Riemannian manifolds. The detailed introduction is given in Section~\ref{sec:SPD recursive layer}.

To simplify the metric computation of the SPD features, we design the diagonalizing layer, which is able to compute log-Euclidean metric efficiently without high-computational SVD operations. With the proposed diagonalizing layer, calculating matrix logarithm is converted to calculate logarithm of scalar elements, and thus we do not require any SVD of matrices. More details can be found in Section~\ref{sec:diag}.

Finally the feature maps are flattened into vectors and passed through a fully connected layer followed by a softmax layer, please see Section~\ref{sec:loss}. All above these are integrated into a fully end-to-end neural network for training as well as testing.

\subsection{Local SPD convolutional layer}\label{sec:spd_cnn}

To describe more clear, we first give the definition of multi-channel SPD matrix.
\newtheorem{Definition}{\textbf{Definition}}
\begin{Definition}
Given a multi-channel matrix $\mathbf{X}\in \mathbb{R}^{C \times D  \times D}$, where $C$ is the channel number and $D$ is the spatial dimension. If each sliced channel $\mathbf{X}^{(i)} \in \mathbb{R}^{D  \times D} (i=1,\cdots,C)$ is SPD, then we call the multi-channel matrix $\mathbf{X}$ be SPD.
\end{Definition}

As the input of network, we may construct multi-channel SPD maps by using covariance or its variants. Then we expect to do local convolution filtering like the standard CNN. To preserve the property of manifold, local convolutional kernels are designed to satisfy the condition of SPD. A theoretical guarantee is given below:
\newtheorem{theorem}{\textbf{Theorem}}
\begin{theorem}\label{thm:mkconv}
Given a multi-channel SPD matrix $\mathbf{X}\in \mathbb{R}^{C \times D  \times D}$. Let $\mathbf{W}\in\mathbb{R}^{C'\times C \times K  \times K}$ be convolutional kernels, where $C'$ is the kernel number and $k$ is the kernel size. Then the convolutional operation is defined as
\begin{equation} \label{eq1_1}	
      F^{(m)}_{i,j}= \sum_{c=0}^{C}\sum_{p=0}^{K-1}\sum_{q=0}^{K-1}W_{p,q}^{(m,c)}X^{(c)}_{i+p,j+q}, m=1,\cdots,C',
\end{equation}
where $\mathbf{F}\in \mathbb{R}^{C' \times (D-K+1)  \times (D-K+1)}$ is the filtering output. If each convolutional kernel $\mathbf{W}^{(m)}\in\mathbb{R}^{C \times K  \times K}$ is multi-channel SPD, then the output map $\mathbf{F}$ is also multi-channel SPD.
\end{theorem}
\begin{proof}
Please see Appendix~\ref{APP_B}.
\end{proof}

To construct an SPD convolutional kernel, we employ the multiplication of one matrix $\mathbf{V}^{m,c}\in \mathbb{R}^{D  \times D}$, formally,
\begin{eqnarray} \label{eq1_2}	
       \mathbf{W}^{m,c}=(\mathbf{V}^{m,c})^T(\mathbf{V}^{m,c})+\epsilon\mathbf{I},
\end{eqnarray}
where $\epsilon\rightarrow 0^+$ and $\mathbf{I}$ is an identity matrix. Obviously, the constructed $\mathbf{W}$ is SPD. Hence, during network learning, we only need to learn the parameter $\mathbf{V}$, and perform Eqn.~(\ref{eq1_2}) to implement the SPD convolution.

For the non-linear activate function, we may employ the element-wise operation on some specific functions, which are proven to be SPD transformations below.
\begin{theorem}\label{thm:activ}
Given an SPD matrix, the transformation of element-wise activation with $exp(\cdot), sinh(\cdot)$ or $cosh(\cdot)$ is still SPD.
\end{theorem}
\begin{proof}
Please see Appendix~ \ref{APP_C}.
\end{proof}

According to Theorem~\ref{thm:activ}, we can implement non-linear transformations without the high-computational SVD, which is often used in the previous methods including the recent work~\cite{huang2017riemannian}. The convolutional filtering and the non-linear activation form the basic local SPD convolutional layer. The SPD convolutional layer is rather flexible in efficient computation, which can be directly implemented by the conventional matrix operation like the standard convolutional layer.

\subsection{SPD Recursive layer}\label{sec:SPD recursive layer}

Inspired by the philosophy of the classic gated recurrent unit (GRU)~\cite{cho2014learning}, we design the manifold-preserved recursive layer as follows:
\begin{eqnarray}
 \label{gru_1}    \mathbf{R}_t= \sigma_g(\mathbf{W}_{fr}^T\mathbf{F}_t\mathbf{W}_{fr}+\mathbf{W}_{hr}^T\mathbf{H}_{t-1}\mathbf{W}_{hr}+b_r+\epsilon*\mathbf{I}), \\
  \label{gru_2}   \mathbf{Z}_t= \sigma_g(\mathbf{W}_{fz}^T\mathbf{F}_t\mathbf{W}_{fz}+\mathbf{W}_{hz}^T\mathbf{H}_{t-1}\mathbf{W}_{hz}+b_z+\epsilon*\mathbf{I}),\\
 \label{gru_3}    \widetilde{\mathbf{H}}_t= sinh(\mathbf{W}_{fh}^T\mathbf{F}_t\mathbf{W}_{fh}+\mathbf{H}_{t-1} \odot \mathbf{R}_t+b_h+\epsilon*\mathbf{I}),~~\\
 \label{gru_4}   \mathbf{H}_t=  \mathbf{Z}_t \odot \mathbf{H}_{t-1}+\widetilde{\mathbf{H}}_t,~~~~~~~~~~~~~~~~~~~~~~~~\\
 \nonumber ~s.t.~b_r\geq0,~b_z\geq0,~ b_h\geq0, \epsilon>0,
\end{eqnarray}
where
 \begin{eqnarray}
 \nonumber    \sigma_g(\mathbf{X})=\frac{exp(\mathbf{X})}{max( exp(\mathbf{X}))}.~~~~~~~~~~~~
\end{eqnarray}
In SPD recursive layer, the projection matrices denoted as $\mathbf{W}_{fr}, \mathbf{W}_{hr}, \mathbf{W}_{hz}, \mathbf{W}_{fr}$ and $\mathbf{W}_{fh}$ are trainable parameters for generating desirable hidden states through bilinear projection. The nonlinear activation function denoted as $\sigma_g(\cdot)$ is designed to generate two gates, denoted as $\mathbf{R}_t,  \mathbf{Z}_t$, with manifold preservation. $\mathbf{R}_t,  \mathbf{Z}_t$ have values ranging in [0, 1] and decide whether to memorize the previous output states through Hadamard product.  Another nonlinear activation function, $sinh$, is employed to endow flexibility to current hidden state. $b_r, b_z$ an $b_h$ are trainable biases of positive values and  $\mathbf{H}_t$ denotes the current output state. Thus, SPD recursive layer is able to well model the temporal dependencies on manifold  by properly memorizing or forgetting the hidden states.

Moreover, we provide a theoretical guarantee for manifold preservation of the  SPD recursive layer:
\begin{theorem}\label{thm:SPD recursive layer}
Given sequential SPD feature maps denoted as $\mathbf{F}_1, \cdots, \mathbf{F}_T$  where $T$ is the temporal length, the defined model above Eqn.~\ref{gru_1}$\sim$~\ref{gru_4} is manifold-preserved.
\end{theorem}
\begin{proof}
Please see Appendix~\ref{APP_D}.
\end{proof}

\subsection{Diagonalizing layer}\label{sec:diag}
The conventional log-Euclidean metric computation~\cite{arsigny2006log} need transform SPD matrices on manifold into points of general Euclidean space so that the conventional metrics can be utilized. Formally, given an SPD map $\mathbf{Z}\in\mathbb{R}^{D\times D}$, the transformation is defined:
\begin{eqnarray} \label{eq_2_1}	
     log( \mathbf{Z})=\mathbf{U}^Tlog(\mathbf{\Sigma})\mathbf{U}, ~~\text{s.t.,}~\mathbf{Z}=\mathbf{U}^T\mathbf{\Sigma}\mathbf{U},
\end{eqnarray}
where $\mathbf{U},~\mathbf{\Sigma}$ are eigenvectors and eigenvalues of $\mathbf{Z}$, and $log(\mathbf{\Sigma})$ is the diagonal matrix of the eigenvalue logarithms.

The key issue is how to bypass high-computational SVD during metric computation. To the end, we develop an elegant trick to map a  Riemannian manifold denoted as $\mathcal{M}_1$ in size of  $D\times D$ into a diagonal  Riemannian manifold $\mathcal{M}_2$ in size of  $D^2\times D^2$. Concretely, two main operations are sequentially performed, which include non-linear positive activation $\delta$ and matrix diagonalizing. Formally,
\begin{eqnarray} \label{eq_3}	
       \mathbf{D}=\mathfrak{D}(\mathbf{Z})=diag(flatten(\delta( \mathbf{Z}))),
\end{eqnarray}
where $\mathbf{D}$ is the output of diagonalizing layer. The standard non-linear activation $\delta$ on the input SPD matrix $\mathbf{Z}$ is used to produce positive activated values. Here the elementwise $exp(\cdot)$ function may be employed. After the non-linear transformation, we flatten the response matrix into a vector and diagonalize it into a diagonal matrix, where the diagonal elements take this vector. Until now, we have done the conversion from a general SPD matrix on $\mathcal{M}_1$ to a specific SPD one, which still lies on a Riemannian manifold denoted as  $\mathcal{M}_2$ due to its SPD property. But for the diagonal SPD matrix, the matrix logarithm only need perform the general element-wise log operation on each diagonal element. Obviously, there are two advantages in the above distance computation: i) completely bypass SVD; ii) efficiently compute without the true diagonalization in the implementation due to the only requirement of calculating on non-zero elements. That means, the diagonalizing operation does not increase the use of memory size while reducing the computation cost.

\subsection{The cross-entropy loss}\label{sec:loss}
After the diagonalizing layer, the discriminative feature maps of dynamic sequence are obtained. To remove those uninformative zeros values, we vectorize these diagonal feature maps with only non-zero elements and then pass them through a full connection layer and a softmax layer. Finally, we use cross entropy loss defined as follows to represent the objective loss function, which can be written as
\begin{eqnarray} \label{eq7_1}
 E =-\sum_{j=1}^N\sum_{c=1}^C\tau(y_j, c) \times log P(c|\mathbf{X}_j),
\end{eqnarray}
\begin{equation} \label{eq7_2}
      \nonumber   \tau(x,y)=
      \begin{cases}
         1,  &  \mbox{if} ~x=y; \\
         0,  &  \mbox{otherwise.}
       \end{cases},~~~
\end{equation}
where $\mathbf{X}_j$ represents the $j$-th training sample of the training set, $P(c|\mathbf{X}_j)$ denotes the probability for the input $\mathbf{X}_j$ being predicted as the $c$-th class, $E$ denotes the cross entropy loss, $y_{i}$ is the label of the $i$th training sample.

\section{Experiments}\label{Experiments}
 We evaluate our DMT-Net by conducting experiments on three action recognition datasets, i.e., the Florence 3D Actions dataset~\cite{seidenari2013recognizing}, HDM05 database~\cite{muller2007documentation} and the Large Scale Combined (LSC) dataset~\cite{zhanglarge}, respectively. We firstly introduce the three datasets. Then, we show the implemental details including the preprocessing process, the SPD feature extrction, and the parameter settings. Finally, we compare the experimental results with the state-of-the-art methods.

\subsection{Datasets}\label{Dataset}
The Florence 3D Actions dataset contains 9 activities:  `wave (WV)', `drink from a bottle (DB)', `answer phone (AP)', `clap (CL)', `tight lace (TL)', `sit down (SD)', `stand up (SU)', `read watch (RW)' and `bow(BO)'. These actions are performed by 10 subjects  for 2 or 3 times yielding 215 activity samples in total and represented by 15 joints without depth data. Similar actions such as `drink from a bottle' and  `answer phone' are easy to be confused.

The HDM05 dataset contains 2,273 sequences of 130 motion classes executed by five actors named `bd', `bk', `dg', `mm' and `tr', and each class has 10 to 50 realizations. The skeleton data contains  31 joints which are much more than those in Florence datasets. The intra-class variations and high number of classes make this dataset more challenging.

The LSC dataset is created by combing nine existing public datasets with both red, green and blue (RGB) video and depth information. In total, it contains 4953 video sequences of 94 action classes performed by 107 subjects. As these video sequences come from different individual datasets, the variations with respect to subjects, performing manners and backgrounds are very large. Moreover,  the number of samples for each action is different. All these factors, i.e. the large size, the large variations and the data imbalance for each class, make LSC dataset more challenging for recognition.

\subsection{Implemental details}\label{details}
In this section, we will specify the implemental procedures of the proposed method in details, which consist of the following parts:

\subsubsection{Preprocessing:}\label{preprocessing}
Before we extract SPD descriptors from skeleton data, we conduct preprocessing to the skeleton data to reduce the variations, such as body orientation variation, body scale variation and so on. This process is done by the following three steps:

1) We downsample action sequences to a unified length by splitting them into a  fixed number of subsequences and then randomly choose one frame from each subsequence.

2) We randomly scale the skeletons with different factors ranging in [0.95, 1.05] to improve the adaptive scaling capacity.

3) We randomly rotate the skeletons along   $x,y$ and $z$ axis with angles ranging in [-45, 45] during training stage, which make the model be robust to orientation variation. Fig.~\ref{rotation} shows this process with two actions of Florence dataset.

\subsubsection{SPD feature extraction:}\label{SPD feature extrction}
The process of extracting  spatio-temporal SPD features  from skeleton data is rather simple. Let $\mathbf{x}_t \in \mathbb{R}^{N_j\times3}$  represent the joint locations of the $t$-th frame where $N_j$ is the number of joints, then the SPD feature of the $t$-th denoted as $\mathbf{X}_t, (t=1,\cdots,T)$ can be calculated as follows:
\begin{equation} \label{eq8_1}
    \mathbf{X}_t=(\mathbf{x}_t-\mathbf{\bar x})(\mathbf{x}_t-\mathbf{\bar x})^{'},
\end{equation}
where $'$ means the transpose operator and $\mathbf{\bar s}$ can be calculated as follows
\begin{equation} \label{eq8_2}
\mathbf{\bar x}=\sum_{t=1}^{T}{\mathbf{x}_t }.
\end{equation}

\subsubsection{Parameter setting:}\label{parameter setting}
In preprocessing process, the sequences are split into 12 subsequences for all datasets. For recognition, we employ the same architecture of DMT-Net when evaluating all the datasets, where the DMT-Net contains  two SPD convolutional layers, one SPD recursive layer, one diagonalizing layer, one fully connected layer and one softmax layer. For SPD convolutional layers, we double the number of SPD convolutional kernels in the second convolutional layer comparing to the first layer. For  SPD recursive layer, the sizes of matrices in each channel of the hidden states are all set to be $9\times9$. The numbers of nodes in the fully connected layer are set to be 800.

\begin{figure}[!t]
\centering
\includegraphics[width=3.2in,,height=2.5in]{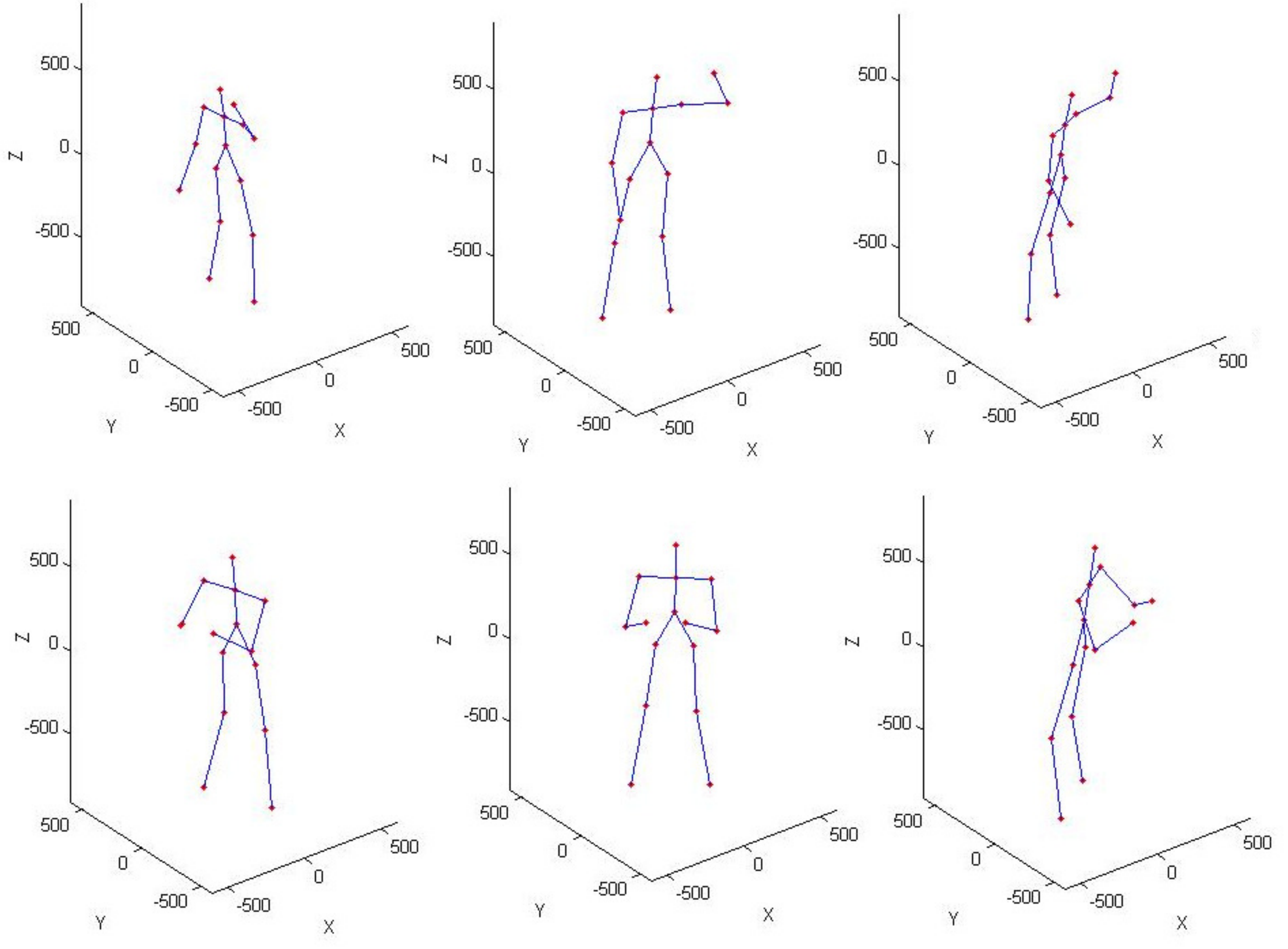}
\caption{Detail of rotation around $z$ axis. The first row contains rotated samples of action 'wave' and the second row contains rotated samples of action 'clap'.}
\label{rotation}
\end{figure}

\subsection{Experiments on Florence dataset}\label{Florence}

\begin{table}[!t]
\renewcommand{\arraystretch}{1.3}
\begin{center}
\begin{tabular}{|c|c|}
\hline
Method & Accuracy ($\%$) \\
\hline
\tabincell{c}{Multi-part Bag-of-Poses ~\cite{seidenari2013recognizing}} & 82.00 \\
\hline
\tabincell{c}{Lie group~\cite{vemulapalli2014human}} & 90.08 \\
\hline
 \tabincell{c}{Shape Analysis on Manifold ~\cite{devanne20153}} & 87.04  \\
\hline
Elastic Function Coding~\cite{anirudh2015elastic} & 89.67 \\
\hline
Graph Based Representation~\cite{wang2016graph} & 91.63 \\
\hline
 \tabincell{c}{Multi-instance Multitask Learning~\cite{yang2017discriminative}} & 95.35 \\
\hline
 \tabincell{c}{Tensor Representation~\cite{koniusz2016tensor}} & 95.47 \\
\hline
 DMT-Net   &  \textbf{98.14}\\
\hline
\end{tabular}
\end{center}
\caption{The comparisons on Florence dataset.}
\label{Label_Florence}
\end{table}

\begin{figure}[!t]
  \centering
  \includegraphics[height=2in, width=2.4in]{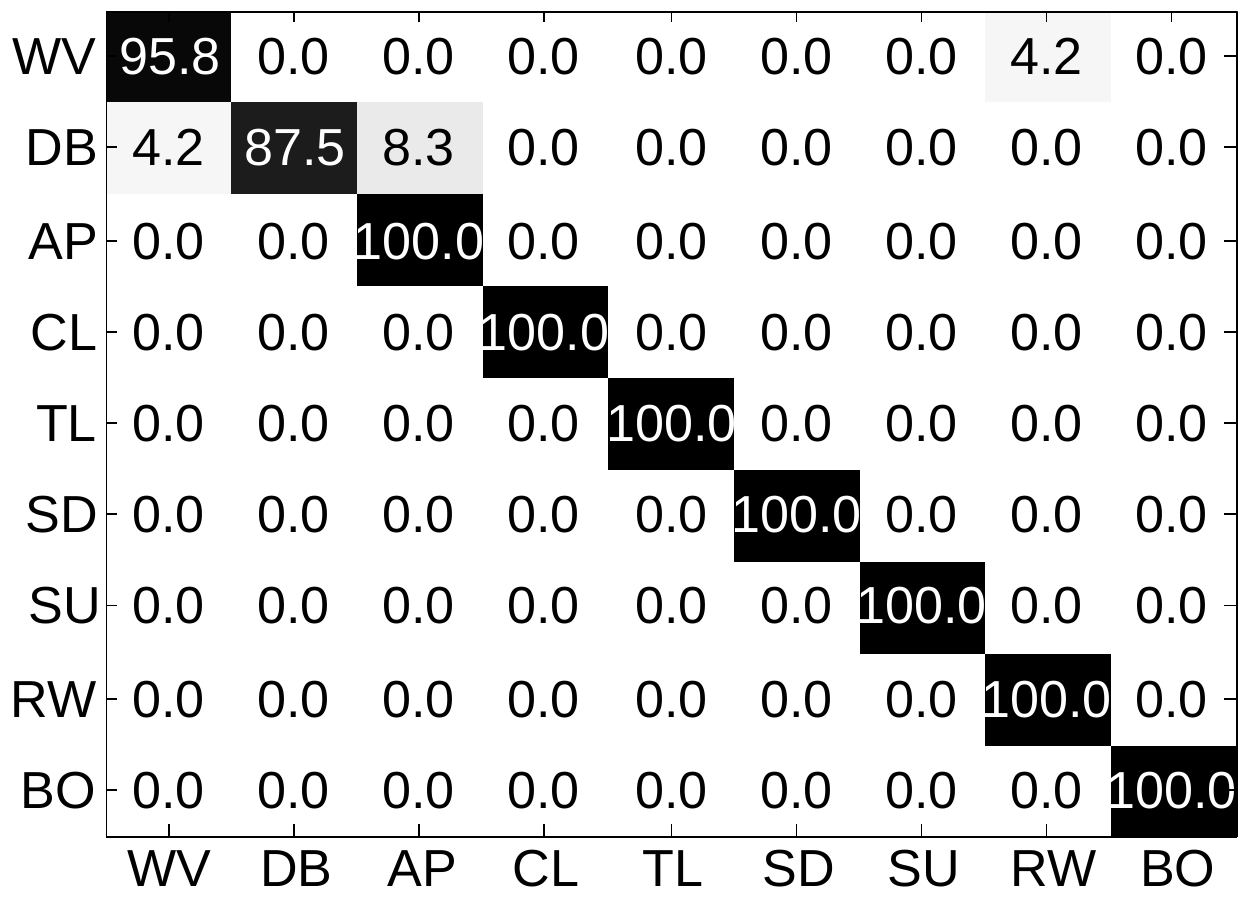}
  \caption{The experimental results of confusion matrix on Florence dataset.}
  \label{Florence_conf}
\end{figure}

For Florence dataset, we strictly follow the protocol of  leave-one-subject-out validation as \cite{wang2016graph}, where skeleton data of nine subjects is used for training and the resting part is for testing. The sizes of the SPD filters in the two SPD convolutional layer are set to be $4\times1\times6\times6$ and  $8\times4\times3\times3$ respectively. The result of the proposed DMT-Net dataset on Florence dataset is shown in Table~\ref{Label_Florence} and  compared with various existed algorithms, such as ~\cite{vemulapalli2014human,anirudh2015elastic,wang2016graph}. Among these compared literatures, the algorithm proposed in~\cite{koniusz2016tensor} achieves  the highest accuracy of  95.47$\%$,  while the accuracy of our proposed DMT-Net is 98.14$\%$ which is almost 3 percent higher.

Fig.~\ref{Florence_conf} shows the confusion matrices of the different recognition results of actions. Our proposed model performs well on seven actions with  100$\%$ recognition rates. Main confusion appears between two pairs of actions which are `wave' versus `drink from a bottle' and  `drink from a bottle' versus `answer phone'. Intuitively, this confusion is reasonable as the two pairs of actions are really similar.

\subsection{Experiments on HDM05 dataset}\label{HDM05}
To achieve a comprehensive comparison to the state-of-the-art methods on  HDM05 dataset, we conduct experiments by following two different protocols employed in previous literatures. For the first protocol  employed in~\cite{wang2015beyond}, actions performed by two subjects named `bd' and `mm' are used for training and the remaining samples are for testing; for the second protocol employed in~\cite{huang2017riemannian}, 10 random evaluations are conducted and for each evaluation a half of the samples of each class  are randomly selected for training and the rests for testing. The sizes of the SPD convolutional kernels in the two SPD convolutional layer are set to be $8\times1\times6\times6$ and $16\times8\times3\times3$ respectively.

The comparisons on HDM05 dataset are shown in Table \ref{Label_HDM05}. For both protocols, our DMT-Net achieves the  best performance comparing to the-state-of-the-art methods. Specifically, for the second protocol, DMT-Net is compared with a Riemannian network proposed in~\cite{huang2017riemannian} which is another kind of deep network on Riemannian manifold. The proposed DMT-Net outperforms the Riemannian network with the recognition rate of 83.24$\%$, which is almost 12 percent higher.

\begin{table}
\renewcommand{\arraystretch}{1.3}
\begin{center}
\begin{tabular}{|c|c|c|}
\hline
Method & \tabincell{c}{Protocol $\#$1\\ Accuracy ($\%$)} & \tabincell{c}{Protocol $\#2$\\ Accuracy($\%$) }  \\
\hline
\tabincell{c}{RSR-ML  \\~\cite{harandi2014manifold}} & 40.00 & - \\
\hline
\tabincell{c}{Cov-RP~\cite{tuzel2006region}} & 58.90 &  -\\
\hline
\tabincell{c}{Ker-RP-RBF~\cite{wang2015beyond}} & 66.20 & -\\
\hline
\tabincell{c}{Lie group  \cite{vemulapalli2014human}} & - &70.26 $\pm$ 2.89\\
\hline
\tabincell{c}{LieNet~\cite{huang2016deep}} & - & 75.78 $\pm$ 2.26\\
\hline
\tabincell{c}{SPDNet~\cite{huang2017riemannian}} &-& 61.45 $\pm$ 1.12\\
\hline
 DMT-Net   &  \textbf{76.25} &\textbf{83.24 $\pm$ 1.52}\\
\hline
\end{tabular}
\end{center}
\caption{The comparisons on HDM05 dataset.}
\label{Label_HDM05}
\end{table}


\begin{table}[h]
\renewcommand{\arraystretch}{1.3}
\begin{center}
{\begin{tabular}{|c|c|c|c|}
\hline
Protocol & Method & ~~~~Precision ($\%$)~~~~ & ~~Recall ($\%$)~~ \\
\hline
\multirow{3}{*}{RCSam}&\tabincell{c}{HON4D ~\cite{oreifej2013hon4d}}  & 84.6& 84.1 \\ \cline{2-4}
{} &\tabincell{c}{Dynamic \\ skeleton~\cite{zhanglarge}} & 85.9 &   85.6 \\ \cline{2-4}
{} &DMT-Net   &  \textbf{87.0}  & 85.1 \\ \hline
\multirow{3}{*}{RCSub}&\tabincell{c}{HON4D ~\cite{oreifej2013hon4d}}& 63.1 & 59.3 \\ \cline{2-4}
{} &\tabincell{c}{Dynamic \\ skeleton ~\cite{zhanglarge}} & 74.5 &  73.7 \\ \cline{2-4}
{} &DMT-Net  &  \textbf{81.0} &\textbf{78.5} \\ \hline
\end{tabular}}
\end{center}
\caption{The comparisons on LSC dataset following RCSam and RCSub protocols.}
\label{LSC_result}
\end{table}

\subsection{Experiments on Large Scale Combined dataset}\label{LSC}
We conduct experiments by LSC dataset by following two different protocols employed in~\cite{zhanglarge}. For the first protocol named  Random Cross Sample (RCSam) using data of 88 action classes,  half of the samples of each class are randomly selected as training data while the rests are used as testing data.  For the second protocol named  Random Cross subject (RCSub) using data of  88 action classes,  half of the subjects are randomly selected as training data and the rest subjects are used as test data. In both protocols, only skeleton data are used for recognition. Due the imbalance of samples in each class,  the values of  precision and  recall are employed for evaluating the performance instead of accuracy. The sizes of the SPD convolutional kernels in the two SPD convolutional layer are set to be $8\times1\times6\times6$ and $16\times8\times3\times3$.


The comparisons on LSC dataset are shown in Table~\ref{LSC_result}. For both protocols, our DMT-Net achieves the  best precisions comparing to the-state-of-the-art methods. Specifically,  for RCSub protocol, DMT-Net achieves relatively better performance: the values of precision and recall are almost 7 and 5 percent high than~\cite{zhanglarge}, which means DMT-Net is more robust to the variation caused by differences of subjects.

\begin{table}[!t]
\renewcommand{\arraystretch}{1.3}
\begin{center}
\begin{tabular}{|c|c|}
\hline
Method & Accuracy ($\%$) \\
\hline
two 1d-CNN layers + one GRU layer (Euclidean space) & 70.27 \\
\hline
DMT-Net without SPD recursive layer   & 71.74 \\
\hline
DMT-Net without SPD convolutional layer  & 67.22  \\
\hline
DMT-Net & \textbf{76.25} \\
\hline
\end{tabular}
\end{center}
\caption{The architecture evaluation with  HDM05 dataset using protocol $\#$1.}
\label{Label_Evaluation}
\end{table}

\subsection{Analysis of DMT-Net}\label{analysis}
In this section, we will conduct additional experiments to verify the effectiveness of designed layers. The following three baseline experiments are respectively conducted for this purpose:
\begin{enumerate}
\item[(1)] Manifold representation vs. Euclidean representation. For Euclidean representation, we take the similar network structure with DMT-Net. The network in Euclidean space contains two 1d-convolutional layers and one GRU layer, where raw features of joint locations (rather than covariance features) are input into the network.
\item[(2)] DMT-Net vs. DMT-Net without SPD recursive layer. To verify the effectiveness of SPD recursive layer, we construct a network by removing SPD recursive layer from  DMT-Net. The sizes of SPD filters in the two convolutional layers are set as same as DMT-Net, which are $8\times1\times6\times6$ and $16\times8\times3\times3$ respectively.
\item[(3)] DMT-Net vs. DMT-Net without SPD convolutional layer. We remove the two SPD convolutional layers from DMT-Net so that the resulting network contains one SPD recursive layer. The size of the hidden state in SPD recursive layer is $20\times$20.
\end{enumerate}
We evaluate these networks on HDM05 dataset with its protocol $\#$1. The results are shown in Table~\ref{Label_Evaluation}. From the results we can have the following observations:
\begin{enumerate}
\item[(i)] Convolutional filtering plays a crucial role in the performance promotion like the standard CNN in Euclidean space. In manifold space composed of covariance matrices, the local convolution filtering should more extract some bundling features co-occurred for certain task.
\item[(ii)] The SPD recursive layer can further improve the performance due to the introduction of sequence dynamics. According to the performances of DMT-Net vs DMT-Net without SPD recursive layer, the accuracy of  DMT-Net is almost 9 percent higher than the network without recursive layer.
\item[(iii)] The Euclidean representation deteriorate the performance compared to manifold representation. The possible reason should be the relationship dependency is more important and robust than raw features in the representation of specific body skeleton data ( even extended into dynamic graph data).
\end{enumerate}

\section{Conclusion}
In this paper, we proposed a novel framework named DMT-Net to model the spatio-temporal dynamic sequences. We segmented the entire sequence into several clips, each of which is described with one SPD matrix.  Since SPD matrices are embedded on Riemannian manifold, we specifically designed a series of novel layers to guarantee matrix transformations still flow on manifolds. The constructed layers contain SPD convolutional, non-linear activation, SPD recursive and the diagonalizing layer. All these layers do not need SVD operation which has high-expensive computation cost. The whole design is generic for the representation learning of manifolds, thus has a constructive value to deep learning and manifold learning.  We all conducted experiments on the task of skeleton-based action recognition, and achieved the state-of-the-art performance under the same experimental environments. In our future work, we will explore more applications for our proposed network framework.

\appendices
\section{Preliminary}\label{APP_A}
\begin{theorem}\label{spd:sigconv}
Given an SPD matrix $\mathbf{X}\in \mathbb{R}^{D  \times D}$. Let $\mathbf{W}\in\mathbb{R}^{K  \times K}$ be a convolutional kernel, then the convolutional operation is defined as
\begin{equation} \label{eqA1_1}	
      O_{i,j}= \sum_{p=0}^{K-1}\sum_{q=0}^{K-1}W_{p,q}X_{i+p,j+q},
\end{equation}
where $\mathbf{O}\in \mathbb{R}^{(D-K+1)  \times (D-K+1)}$ is the filtering output. If the convolutional kernel $\mathbf{W}$ is SPD, then the output map $\mathbf{O}$ is also SPD.
\end{theorem}
\begin{proof}
As $\mathbf{W}$ is SPD, it can  be decomposed into:
\begin{eqnarray}
   \mathbf{W}= \mathbf{H} \mathbf{H}^T,
\end{eqnarray}
where $\mathbf{H}=[\mathbf{h}_1, \mathbf{h}_2, \cdots,  \mathbf{h}_K]$ is a matrix of full rank. Then the convolutional result of an  SPD representation matrix $\mathbf{X} \in \mathbb{R}^{D \times D}$  can be written as
\begin{eqnarray}
    \mathbf{O}=\mathbf{X}*\mathbf{W}= \mathbf{X}*(\mathbf{H} \mathbf{H}^T) ~~~~~~~~~~~~~~~~~~~~~~~~~~~~~~~~~~~~~~~~~\\
 \label{Ae_3}   =\mathbf{X}*(\mathbf{h}_1\mathbf{h}_1^T)+\cdots+\mathbf{X}*(\mathbf{h}_K\mathbf{h}_K^T) ~~\\
 \label{Ae_4}  =\mathbf{X}*\mathbf{h}_1*\mathbf{h}_1^T+\cdots+\mathbf{X}*\mathbf{h}_K*\mathbf{h}_K^T,
\end{eqnarray}
where the derivation from Eqn.~(\ref{Ae_3}) to Eqn.~(\ref{Ae_4}) uses the property of separable convolution. Suppose that $\mathbf{h}_i=[n_1, n_2, \cdots,n_D]^T, ~i=[1,2,\cdots,D]$, the convolution between $\mathbf{X}$ and $\mathbf{h}_i$ can be written as:\\
      $~~~~~~~~~~~~~~\mathbf{X}*\mathbf{h}_i=\mathbf{G}_{\mathbf{h}_i}\mathbf{X},~~~\mathbf{X}*\mathbf{h}_i^T=\mathbf{X}\mathbf{G}_{\mathbf{h}_i}^T$ ,\\
where
     $\mathbf{G}_{\mathbf{h}_i}\in\mathbb{R}^{(N-D+1) \times N}$  and
\begin{eqnarray}
 ~~~~~~~~~~ \mathbf{G}_{\mathbf{h}_i}=  \begin{bmatrix}
                               n_1, ~ n_2,  ~      \cdots,     ~    n_D,  ~      0,     ~      0, ~ 0,~ \cdots, ~ 0\\
                            ~     0,    ~ ~   n _1,         n_2,   ~     \cdots,  ~    n_D,        0, ~ 0, ~ \cdots,~  0\\
                            ~     0,  ~ ~    0,         ~       n_1,        n_2,   ~   \cdots,   ~    n_D ,  0,~   \cdots,~  0\\
                                                                  \cdots    \\
                            ~     0,   ~ ~       0,       ~   \cdots,   ~ ~ 0,      ~    n_1, ~ ~  n_2,  ~  \cdots,   ~ ~  n_D\\
 \end{bmatrix}.
\end{eqnarray}
Then we get the following equations:
\begin{eqnarray}
           \mathbf{X}*\mathbf{h}_i*\mathbf{h}_i^T= \mathbf{G}_{\mathbf{h}_i}\mathbf{X}\mathbf{G}_{\mathbf{h}_i}^T,
\end{eqnarray}
and
\begin{eqnarray}
           \mathbf{O}= \mathbf{X}*\mathbf{H} = \mathbf{G}_{\mathbf{h}_1}\mathbf{X}\mathbf{G}_{\mathbf{h}_1}^T+\cdots+\mathbf{G}_{\mathbf{h}_D}\mathbf{X}\mathbf{G}_{\mathbf{h}_D}^T.
\end{eqnarray}
As the rank of $\mathbf{G}_{\mathbf{h}_i}$ equals $N-D+1$, $\mathbf{G}_{\mathbf{h}_i}\mathbf{X}\mathbf{G}_{\mathbf{h}_i}^T$ is also an SPD matrix. Thus $\forall$ $\mathbf{z}\in\mathbb{R}^{N}, \mathbf{z}\neq\mathbf{0}$,  we have
\begin{eqnarray}
        \mathbf{z}^T\mathbf{O}\mathbf{z}=\sum_{i=1}^D\mathbf{z}^T\mathbf{G}_{\mathbf{h}_i}\mathbf{X}\mathbf{G}_{\mathbf{h}_i}^T\mathbf{z}>0.
\end{eqnarray}
So $\mathbf{O}$ is an SPD matrix.
 \end{proof}

\section{Proof of Theorem \ref{thm:mkconv}}\label{APP_B}
\begin{proof}
The $m$-th channel of  $\mathbf{F}$ can  be written as:
\begin{eqnarray}
   \mathbf{F}^{(m)}= \sum_{c=1}^C\mathbf{X}^{(c)}*\mathbf{W}^{(m, c)},
\end{eqnarray}
where $\mathbf{X}^{(c)}$ denotes the $c$-th channel of  input descriptor, apparently $\mathbf{X}^{(c)}$ and  $\mathbf{W}^{(m, c)}$ are SPD matrices. According to Theorem~\ref{spd:sigconv}, $\mathbf{F}^{(m)}$ is an SPD matrix.
So  $\mathbf{F}$ is also an multi-channel SPD matrix.
 \end{proof}

\section{Proof of Theorem \ref{thm:activ}}\label{APP_C}
\begin{proof}
We take $exp(\cdot)$ as an example. Let $\mathbf{X} =[X_{ij}]_{D \times D}$ denote an SPD matrix, then the element-wise activation result can be denoted as:
\begin{eqnarray}
   exp(\mathbf{X})=\begin{bmatrix}
                               e^{X_{11}}, ~ e^{X_{12}},  ~      \cdots,     e^{X_{1D}}\\
                               e^{X_{21}}, ~ e^{X_{22}},  ~      \cdots,     e^{X_{2D}}\\
                                                                  \cdots    \\
                               e^{X_{D1}}, ~ e^{X_{D2}},  ~      \cdots,     e^{X_{DD}}\\
 \end{bmatrix}, ~~~~~~~~~~~~~~~~~~~~
  \\=\begin{bmatrix}
                               \sum_{i=0}^{\infty}\frac{X_{11}^i}{i!}, ~\sum_{i=0}^{\infty}\frac{X_{12}^i}{i!},  ~ \cdots,  \sum_{i=0}^{\infty}\frac{X_{1D}^i}{i!} \\
                               \sum_{i=0}^{\infty}\frac{X_{21}^i}{i!}, ~\sum_{i=0}^{\infty}\frac{X_{22}^i}{i!},  ~ \cdots,  \sum_{i=0}^{\infty}\frac{X_{2D}^i}{i!}\\
                                                                  \cdots    \\
                                \sum_{i=0}^{\infty}\frac{X_{D1}^i}{i!}, ~\sum_{i=0}^{\infty}\frac{X_{D2}^i}{i!}, ~ \cdots,  \sum_{i=0}^{\infty}\frac{X_{DD}^i}{i!}\\
 \end{bmatrix},
\\=\mathbf{1}+\mathbf{X}+\frac{1}{2}\mathbf{X}\circ\mathbf{X}+\frac{1}{3!}\mathbf{X}\circ\mathbf{X}\circ\mathbf{X}+\cdots, ~~~~~~~~
\end{eqnarray}
where $\circ$ means Hadamard product of two matrices. According to the Schur product theorem, $\mathbf{X}\circ\mathbf{X}\cdots\mathbf{X}$ is an SPD matrix.
So $exp(\mathbf{X})$, which equals the summation of multiple positive definite matrices and semi-positive definite matrices, is an SPD matrix. Similarly, we can prove that
\begin{eqnarray}
sinh(\mathbf{X})=\sum_{i=0}^{\infty}\frac{\mathbf{X}^{2i+1}}{(2i+1)!}
\end{eqnarray}
and
\begin{eqnarray}
cosh(\mathbf{X})=\sum_{i=0}^{\infty}\frac{\mathbf{X}^{2i}}{(2i)!}
\end{eqnarray}
are also SPD matrices where  $\mathbf{X}^i$ means element-wise power here.
\end{proof}

\section{Proof of Theorem \ref{thm:SPD recursive layer}}\label{APP_D}
\begin{proof}
SPD recursive layer mainly contains  three kinds of operations which are  bilinear projection with diagonal bias, non-linear activation functions including $exp(\mathbf{\cdot})$ and $sinh(\mathbf{\cdot})$, and Hadamard product. According to  the definition of SPD, Theorem~\ref{thm:activ} and Schur product theorem respectively, these three  operations can be easily proved to preserve symmetric positive definiteness. Moreover, as $exp(\mathbf{\cdot})$ preserves symmetric positive definiteness and $\forall \mathbf{X}$, we have $max( exp(\mathbf{X}))>0$, so $\sigma_g(\mathbf{X})$ also preserves symmetric positive definiteness.

Then  we can prove Theorem~\ref{thm:SPD recursive layer} with mathematical induction. When $t=1$, the initial hidden state denoted as $\mathbf{H}_0$ is set to be zero matrix. Then $\mathbf{R}_1$ can be rewritten as
\begin{eqnarray}
 \mathbf{R}_1= \sigma_g(\mathbf{W}_{fr}^T\mathbf{F}_1\mathbf{W}_{fr}+b_r+\epsilon*\mathbf{I}),~~~\\
 \mathbf{Z}_1= \sigma_g(\mathbf{W}_{fz}^T\mathbf{F}_1\mathbf{W}_{fz}+b_z+\epsilon*\mathbf{I}),~~~\\
\widetilde{\mathbf{H}}_1= sinh(\mathbf{W}_{fh}^T\mathbf{F}_1\mathbf{W}_{fh}+b_h+\epsilon*\mathbf{I}).
 \end{eqnarray}
As $\mathbf{F}_1$ is SPD, then apparently $ \mathbf{R}_1, \mathbf{Z}_1, \widetilde{\mathbf{H}}_1$ are all SPD. Thus  $\mathbf{H}_1=\widetilde{\mathbf{H}}_1$ is also SPD.

Then $\forall t\in[2,\cdots,T]$, if $\mathbf{H}_{t-1}$ is SPD,  similar to the situation of $t=1$, we can also prove that $\mathbf{R}_t$, $\mathbf{Z}_t$, $\widetilde{\mathbf{H}}_t$ and ${\mathbf{H}}_t$ in Eqn.~\ref{gru_1},~\ref{gru_2}, ~\ref{gru_3} and ~\ref{gru_4} are SPD. This is because  all the operations dealing with $\mathbf{H}_{t-1}$ and $\mathbf{F}_t$  preserve symmetric positive definiteness.

Thus, according to mathematical induction,  the output states of SPD recursive layer, i.e. $\mathbf{H}_t, t=1,\cdots,T$,  are  SPD.

\end{proof}


\begin{thebibliography}{10}\itemsep=-1pt
%
\bibitem{amor2016action}
B.~B. Amor, J.~Su, and A.~Srivastava.
\newblock Action recognition using rate-invariant analysis of skeletal shape
  trajectories.
\newblock {\em IEEE transactions on pattern analysis and machine intelligence},
  38(1):1--13, 2016.

\bibitem{anirudh2015elastic}
R.~Anirudh, P.~Turaga, J.~Su, and A.~Srivastava.
\newblock Elastic functional coding of human actions: From vector-fields to
  latent variables.
\newblock In {\em Proceedings of the IEEE Conference on Computer Vision and
  Pattern Recognition}, pages 3147--3155, 2015.

\bibitem{arsigny2006log}
V.~Arsigny, P.~Fillard, X.~Pennec, and N.~Ayache.
\newblock Log-euclidean metrics for fast and simple calculus on diffusion
  tensors.
\newblock {\em Magnetic resonance in medicine}, 56(2):411--421, 2006.

\bibitem{baccouche2011sequential}
M.~Baccouche, F.~Mamalet, C.~Wolf, C.~Garcia, and A.~Baskurt.
\newblock Sequential deep learning for human action recognition.
\newblock In {\em International Workshop on Human Behavior Understanding},
  pages 29--39. Springer, 2011.

\bibitem{belkin2003laplacian}
M.~Belkin and P.~Niyogi.
\newblock Laplacian eigenmaps for dimensionality reduction and data
  representation.
\newblock {\em Neural computation}, 15(6):1373--1396, 2003.

\bibitem{Chaudhry2013Bio}
R.~Chaudhry, F.~Ofli, G.~Kurillo, R.~Bajcsy, and R.~Vidal.
\newblock Bio-inspired dynamic 3d discriminative skeletal features for human
  action recognition.
\newblock 13(4):471--478, 2013.

\bibitem{cho2014learning}
K.~Cho, B.~Van~Merri{\"e}nboer, C.~Gulcehre, D.~Bahdanau, F.~Bougares,
  H.~Schwenk, and Y.~Bengio.
\newblock Learning phrase representations using rnn encoder-decoder for
  statistical machine translation.
\newblock {\em arXiv preprint arXiv:1406.1078}, 2014.

\bibitem{devanne20153}
M.~Devanne, H.~Wannous, S.~Berretti, P.~Pala, M.~Daoudi, and A.~Del~Bimbo.
\newblock 3-d human action recognition by shape analysis of motion trajectories
  on riemannian manifold.
\newblock {\em IEEE transactions on cybernetics}, 45(7):1340--1352, 2015.

\bibitem{du2015hierarchical}
Y.~Du, W.~Wang, and L.~Wang.
\newblock Hierarchical recurrent neural network for skeleton based action
  recognition.
\newblock In {\em Proceedings of the IEEE conference on computer vision and
  pattern recognition}, pages 1110--1118, 2015.

\bibitem{edelman1998geometry}
A.~Edelman, T.~A. Arias, and S.~T. Smith.
\newblock The geometry of algorithms with orthogonality constraints.
\newblock {\em SIAM journal on Matrix Analysis and Applications},
  20(2):303--353, 1998.

\bibitem{evangelidis2014skeletal}
G.~Evangelidis, G.~Singh, and R.~Horaud.
\newblock Skeletal quads: Human action recognition using joint quadruples.
\newblock In {\em Pattern Recognition (ICPR), 2014 22nd International
  Conference on}, pages 4513--4518. IEEE, 2014.

\bibitem{goh2008clustering}
A.~Goh and R.~Vidal.
\newblock Clustering and dimensionality reduction on riemannian manifolds.
\newblock In {\em Computer Vision and Pattern Recognition, 2008. CVPR 2008.
  IEEE Conference on}, pages 1--7. IEEE, 2008.

\bibitem{gowayyed2013histogram}
M.~A. Gowayyed, M.~Torki, M.~E. Hussein, and M.~El-Saban.
\newblock Histogram of oriented displacements (hod): Describing trajectories of
  human joints for action recognition.
\newblock In {\em IJCAI}, 2013.

\bibitem{harandi2014manifold}
M.~T. Harandi, M.~Salzmann, and R.~Hartley.
\newblock From manifold to manifold: Geometry-aware dimensionality reduction
  for spd matrices.
\newblock In {\em European Conference on Computer Vision}, pages 17--32.
  Springer, 2014.

\bibitem{Harandi2012Sparse}
M.~T. Harandi, C.~Sanderson, R.~Hartley, and B.~C. Lovell.
\newblock Sparse coding and dictionary learning for symmetric positive definite
  matrices: a kernel approach.
\newblock In {\em European Conference on Computer Vision}, pages 216--229,
  2012.

\bibitem{huang2017riemannian}
Z.~Huang and L.~J. Van~Gool.
\newblock A riemannian network for spd matrix learning.
\newblock In {\em AAAI}, volume~2, page~6, 2017.

\bibitem{huang2016deep}
Z.~Huang, C.~Wan, T.~Probst, and L.~Van~Gool.
\newblock Deep learning on lie groups for skeleton-based action recognition.
\newblock {\em arXiv preprint arXiv:1612.05877}, 2016.

\bibitem{hussein2013human}
M.~E. Hussein, M.~Torki, M.~A. Gowayyed, and M.~El-Saban.
\newblock Human action recognition using a temporal hierarchy of covariance
  descriptors on 3d joint locations.
\newblock In {\em IJCAI}, volume~13, pages 2466--2472, 2013.

\bibitem{jayasumana2013kernel}
S.~Jayasumana, R.~Hartley, M.~Salzmann, H.~Li, and M.~Harandi.
\newblock Kernel methods on the riemannian manifold of symmetric positive
  definite matrices.
\newblock In {\em Proceedings of the IEEE Conference on Computer Vision and
  Pattern Recognition}, pages 73--80, 2013.

\bibitem{koles1990spatial}
Z.~J. Koles, M.~S. Lazar, and S.~Z. Zhou.
\newblock Spatial patterns underlying population differences in the background
  eeg.
\newblock {\em Brain topography}, 2(4):275--284, 1990.

\bibitem{koniusz2016tensor}
P.~Koniusz, A.~Cherian, and F.~Porikli.
\newblock Tensor representations via kernel linearization for action
  recognition from 3d skeletons.
\newblock In {\em European Conference on Computer Vision}, pages 37--53.
  Springer, 2016.

\bibitem{krizhevsky2012imagenet}
A.~Krizhevsky, I.~Sutskever, and G.~E. Hinton.
\newblock Imagenet classification with deep convolutional neural networks.
\newblock In {\em Advances in neural information processing systems}, pages
  1097--1105, 2012.

\bibitem{muller2007documentation}
M.~M{\"u}ller, T.~R{\"o}der, M.~Clausen, B.~Eberhardt, B.~Kr{\"u}ger, and
  A.~Weber.
\newblock Documentation mocap database hdm05.
\newblock 2007.

\bibitem{Ofli2012Sequence}
F.~Ofli, R.~Chaudhry, G.~Kurillo, R.~Vidal, and R.~Bajcsy.
\newblock Sequence of the most informative joints (smij): A new representation
  for human skeletal action recognition.
\newblock In {\em Computer Vision and Pattern Recognition Workshops}, pages
  8--13, 2012.

\bibitem{oreifej2013hon4d}
O.~Oreifej and Z.~Liu.
\newblock Hon4d: Histogram of oriented 4d normals for activity recognition from
  depth sequences.
\newblock In {\em Proceedings of the IEEE Conference on Computer Vision and
  Pattern Recognition}, pages 716--723, 2013.

\bibitem{pang2008gabor}
Y.~Pang, Y.~Yuan, and X.~Li.
\newblock Gabor-based region covariance matrices for face recognition.
\newblock {\em IEEE Transactions on Circuits and Systems for Video Technology},
  18(7):989--993, 2008.

\bibitem{pennec2006riemannian}
X.~Pennec, P.~Fillard, and N.~Ayache.
\newblock A riemannian framework for tensor computing.
\newblock {\em International Journal of Computer Vision}, 66(1):41--66, 2006.

\bibitem{porikli2006covariance}
F.~Porikli, O.~Tuzel, and P.~Meer.
\newblock Covariance tracking using model update based on lie algebra.
\newblock In {\em Computer Vision and Pattern Recognition, 2006 IEEE Computer
  Society Conference on}, volume~1, pages 728--735. IEEE, 2006.

\bibitem{seidenari2013recognizing}
L.~Seidenari, V.~Varano, S.~Berretti, A.~Bimbo, and P.~Pala.
\newblock Recognizing actions from depth cameras as weakly aligned multi-part
  bag-of-poses.
\newblock In {\em Proceedings of the IEEE Conference on Computer Vision and
  Pattern Recognition Workshops}, pages 479--485, 2013.

\bibitem{tao2015moving}
L.~Tao and R.~Vidal.
\newblock Moving poselets: A discriminative and interpretable skeletal motion
  representation for action recognition.
\newblock In {\em Proceedings of the IEEE International Conference on Computer
  Vision Workshops}, pages 61--69, 2015.

\bibitem{tosato2010multi}
D.~Tosato, M.~Farenzena, M.~Spera, V.~Murino, and M.~Cristani.
\newblock Multi-class classification on riemannian manifolds for video
  surveillance.
\newblock In {\em European conference on computer vision}, pages 378--391.
  Springer, 2010.

\bibitem{tuzel2006region}
O.~Tuzel, F.~Porikli, and P.~Meer.
\newblock Region covariance: A fast descriptor for detection and
  classification.
\newblock {\em Computer Vision--ECCV 2006}, pages 589--600, 2006.

\bibitem{tuzel2008pedestrian}
O.~Tuzel, F.~Porikli, and P.~Meer.
\newblock Pedestrian detection via classification on riemannian manifolds.
\newblock {\em IEEE transactions on pattern analysis and machine intelligence},
  30(10):1713--1727, 2008.

\bibitem{vemulapalli2014human}
R.~Vemulapalli, F.~Arrate, and R.~Chellappa.
\newblock Human action recognition by representing 3d skeletons as points in a
  lie group.
\newblock In {\em Proceedings of the IEEE conference on computer vision and
  pattern recognition}, pages 588--595, 2014.

\bibitem{wang2015beyond}
L.~Wang, J.~Zhang, L.~Zhou, C.~Tang, and W.~Li.
\newblock Beyond covariance: Feature representation with nonlinear kernel
  matrices.
\newblock In {\em Proceedings of the IEEE International Conference on Computer
  Vision}, pages 4570--4578, 2015.

\bibitem{wang2016graph}
P.~Wang, C.~Yuan, W.~Hu, B.~Li, and Y.~Zhang.
\newblock Graph based skeleton motion representation and similarity measurement
  for action recognition.
\newblock In {\em European Conference on Computer Vision}, pages 370--385.
  Springer, 2016.

\bibitem{wang2015discriminant}
W.~Wang, R.~Wang, Z.~Huang, S.~Shan, and X.~Chen.
\newblock Discriminant analysis on riemannian manifold of gaussian
  distributions for face recognition with image sets.
\newblock In {\em Proceedings of the IEEE Conference on Computer Vision and
  Pattern Recognition}, pages 2048--2057, 2015.

\bibitem{xia2012view}
L.~Xia, C.-C. Chen, and J.~Aggarwal.
\newblock View invariant human action recognition using histograms of 3d
  joints.
\newblock In {\em Computer Vision and Pattern Recognition Workshops (CVPRW),
  2012 IEEE Computer Society Conference on}, pages 20--27. IEEE, 2012.

\bibitem{yang2017discriminative}
Y.~Yang, C.~Deng, S.~Gao, W.~Liu, D.~Tao, and X.~Gao.
\newblock Discriminative multi-instance multitask learning for 3d action
  recognition.
\newblock {\em IEEE Transactions on Multimedia}, 19(3):519--529, 2017.

\bibitem{zhanglarge}
J.~Zhang, W.~Li, P.~Wang, P.~Ogunbona, S.~Liu, and C.~Tang.
\newblock A large scale rgb-d dataset for action recognition.
%
\end{thebibliography}

\end{document}